\documentclass{ws-procs11x85}
\usepackage{ws-procs-thm}           
\usepackage{color}
\usepackage{booktabs} 
\usepackage{amsmath}
\usepackage{amssymb}
\usepackage{multirow}
\usepackage{graphicx}
\usepackage{enumitem}
\usepackage{setspace}

\newcommand{\editage}{black}
\newcommand{\fix}{black}
\newcommand{\altcolor}{black}
\newcommand{\revisionOne}{black}
\newcommand{\revisionTwo}{black}
\newcommand{\revisionThree}{black}
\newcommand{\revisionfour}{black}
\newcommand{\editageTwo}{black}
\newcommand{\ho}{black}
\newcommand{\ediThree}{black}

\begin{document}

\title{DNA Steganalysis Using Deep Recurrent Neural Networks}

\author{Ho Bae\,$^{\text{\sf 1}}$, Byunghan Lee\,$^{\text{\sf 2, \sf3}}$, Sunyoung Kwon\,$^{\text{\sf 2, \sf4 }}$ and Sungroh Yoon\,$^{\text{\sf 1, \sf 2, \sf5,}*}$}

\address{$^{\text{\sf 1}}$Interdisciplinary Program in Bioinformatics, Seoul National University, Seoul 08826, Korea \\ 
$^{\text{\sf 2}}$Electrical and Computer Engineering, Seoul National University, Seoul 08826, Korea \\
$^{\text{\sf 3}}$Electronic and IT Media Engineering, Seoul National University of Science and Technology, Seoul 01811, Korea\\
$^{\text{\sf 4}}$Clova AI Research, NAVER Corp., Seongnam 13561, Korea\\
$^{\text{\sf 5}}$ASRI and INMC, Seoul National University, Seoul 08826, Korea\\
$^{}$E-mail: sryoon@snu.ac.kr}

\begin{abstract}

Recent advances \textcolor{\ediThree}{in} next-generation sequencing technologies have \textcolor{\ediThree}{facilitated the use} of deoxyribonucleic acid (DNA) as a novel covert channels in steganography. There \textcolor{\ediThree}{are} various methods \textcolor{\ediThree}{that exist} in other domains to detect hidden messages in conventional covert channels. However, they have not been applied to DNA steganography. The current most common detection \textcolor{\ediThree}{approaches}, \textcolor{\editageTwo}{namely} frequency analysis-based methods, often \textcolor{\editageTwo}{overlook} important signals when directly applied to DNA steganography \textcolor{\editageTwo}{because those} methods depend on the distribution of the number of sequence characters.
To address \textcolor{\editageTwo}{this} limitation, we propose a general sequence learning-based DNA steganalysis framework. The \textcolor{\editageTwo}{proposed} approach learns the intrinsic distribution of coding and non-coding sequences and detects hidden messages by exploiting distribution variations after hiding \textcolor{\ediThree}{these} messages. Using deep recurrent neural networks (RNNs), our framework identifies the distribution variations by using the classification score \textcolor{\ediThree}{to predict} whether a sequence is to be a coding or non-coding sequence.
We compare our \textcolor{\editageTwo}{proposed} method to various existing methods and biological sequence analysis methods implemented on top of our framework. According to our experimental results, our approach delivers a robust detection performance \textcolor{\ho}{compared to other tools.}

\end{abstract}


\copyrightinfo{\copyright\ 2018 The Authors. Open Access chapter published by World Scientific Publishing Company and distributed under the terms of the Creative Commons Attribution Non-Commercial (CC BY-NC) 4.0 License.}

\bodymatter

\section{Introduction}
\textcolor{\ho}{Steganography serves to conceal} \textcolor{\ediThree}{the existence and content of messages} in media \textcolor{\editageTwo}{using various techniques,} \textcolor{\ediThree}{including} \textcolor{\editageTwo}{changing the} pixels in an image{~\cite{bennett2004linguistic}}. 
\textcolor{\ediThree}{Generally, steganography is used to achieve two main goals.}
\textcolor{\ediThree}{On the one hand,} it is used as digital watermarking to protect intellectual property. \textcolor{\ediThree}{On the other hand,} it is used as a covert \textcolor{\editageTwo}{approach} to communicating without \textcolor{\editageTwo}{the possibility of detection} by unintended observers.
In contrast, \textcolor{\ediThree}{steganalysis is the study of detecting} hidden messages.
\textcolor{\ediThree}{Steganalysis also has two main goals, which are}  detection and decryption of hidden messages\textcolor{\revisionOne}{~\cite{bennett2004linguistic, mitras2013proposed}}.

Among the various media \textcolor{\editageTwo}{employed to hide} information,
deo\-xyribonucleic acid (DNA) is appealing \textcolor{\editageTwo}{owing to its \textcolor{\ho}{chemical} stability and,} thus is a suitable candidates as \textcolor{\ediThree}{a} carrier of \textcolor{\ediThree}{concealed} information. 
As a storage medium, DNA \textcolor{\ediThree}{has the capacity to store} large amounts of data \textcolor{\ediThree}{that} exceed the capacity of current storage media{~\cite{beck2012finding}}. For instance, a gram of DNA contains approximately \textcolor{\ho}{$10^{21}$} DNA bases (108 tera-bytes), which \textcolor{\editageTwo}{indicates} that only a few grams of DNA can store all information available{~\cite{gehani2003dna}}. In addition, with the advent of next-generation sequencing, individual genotyping has become \textcolor{\ho}{affordable{~\cite{cordell2005genetic}}}, and DNA in turn has become an appealing covert channels.

\textcolor{\altcolor}{To hide information in a DNA sequence, steganography methods require \textcolor{\editageTwo}{that a} reference target sequence and \textcolor{\ediThree}{a message to} be hidden{~\cite{katzenbeisser2000information}}. A na\"ive example of a substitution-based method for watermarking that exploits the preservation of amino acids is shown in \figurename~\ref{fig:watermarked} \textcolor{\ho}{(see the caption for details)}. The hiding space of this method is restricted to exon regions using a complementary pair that does not interfere with protein translation. However, most DNA steganography methods are designed \textcolor{\ho}{without considering the hiding spaces}, and they change a sequence into binary format \textcolor{\ediThree}{ utilizing} well-known encryption techniques.}

\textcolor{\revisionOne}{In this regard,} Clelland et al.{~\cite{clelland1999hiding}}, first proposed DNA steganography \textcolor{\ediThree}{that utilized} the microdot technique.
Yachie et al.{~\cite{yachie2007alignment}}, demonstrated that living organisms can be used as data storage media by inserting artificial DNA into artificial genomes and using a substitution cipher coding scheme. This technique is reproducible and successfully inserts four watermarks into the cell of a living organism{~\cite{gibson2010creation}}. \textcolor{\revisionTwo}{Several other encoding schemes have been proposed {~\cite{brenner2000vitro, tanaka2005public}}. The DNA-Crypt coding scheme  ~\cite{heider2007dna} translates a message into 5-bit sequences, and the ASCII coding scheme{~\cite{jiao2008code}} translates words into their ASCII representation, converts them from decimals to binary, and then replaces 00 with adenine (\texttt{A}), 01 with cytosine (\texttt{C}), 10 with guanine (\texttt{G}), and 11 with thymine (\texttt{T}).}

\textcolor{\editageTwo}{With the recent advancements with respect to} steganography methods, various steganalysis studies have been \textcolor{\ediThree}{conducted using traditional} storage media.
\textcolor{\editageTwo}{D}etection techniques \textcolor{\editageTwo}{that are} based on statistical analysis, neural networks, and genetic algorithms{~\cite{maitra2011digital}} have been developed for common covert objects \textcolor{\altcolor}{such as} digital images, video, and audio. For example, Bennett{~\cite{bennett2004linguistic}} exploits letter frequency, word frequency, grammar style, semantic continuity, and logical methodologies.
\textcolor{\altcolor}{
However, these conventional steganalysis methods have not been applied to DNA steganography.}

In this paper, we show that conventional steganalysis methods are not directly \textcolor{\ho}{applicable} to DNA steganography. \textcolor{\editageTwo}{Currently, the} most common\textcolor{\editageTwo}{ly employed} detection schemes, \textcolor{\editageTwo}{i.e.,} a statistical hypothesis testing methods,
\textcolor{\editageTwo}{are limited with respect to} the number of input queries \textcolor{\ediThree}{in order} to estimate distribution to perform statistical test{~\cite{grosse2017statistical}}. To overcome the limitations of these existing methods, we propose a DNA steganalysis method based on learning the internal structure of unmodified genome sequences (\textit{i.e.}, intron and exon modeling{~\cite{lee2015boosted,lee2015dna}}) using deep recurrent neural networks (RNNs). 
The \textcolor{\ho}{RNN}-based classifier is used to \textcolor{\ediThree}{identify} modified genome sequences.
In addition, we enhance our proposed model using unsupervised pre-training of a sequence-to-sequence autoencoder \textcolor{\editageTwo}{in order} to overcome the restriction of the robustness of detection performance. Finally, we compared our \textcolor{\editageTwo}{proposed} method to various machine learning-based classifiers and biological sequence analysis methods \textcolor{\editageTwo}{that were} implemented on top of our framework.

\section{Background}
\begin{figure}[t]
	\centering
    \begin{minipage}{.45\textwidth}
    	\centering
        \includegraphics[width=\linewidth]{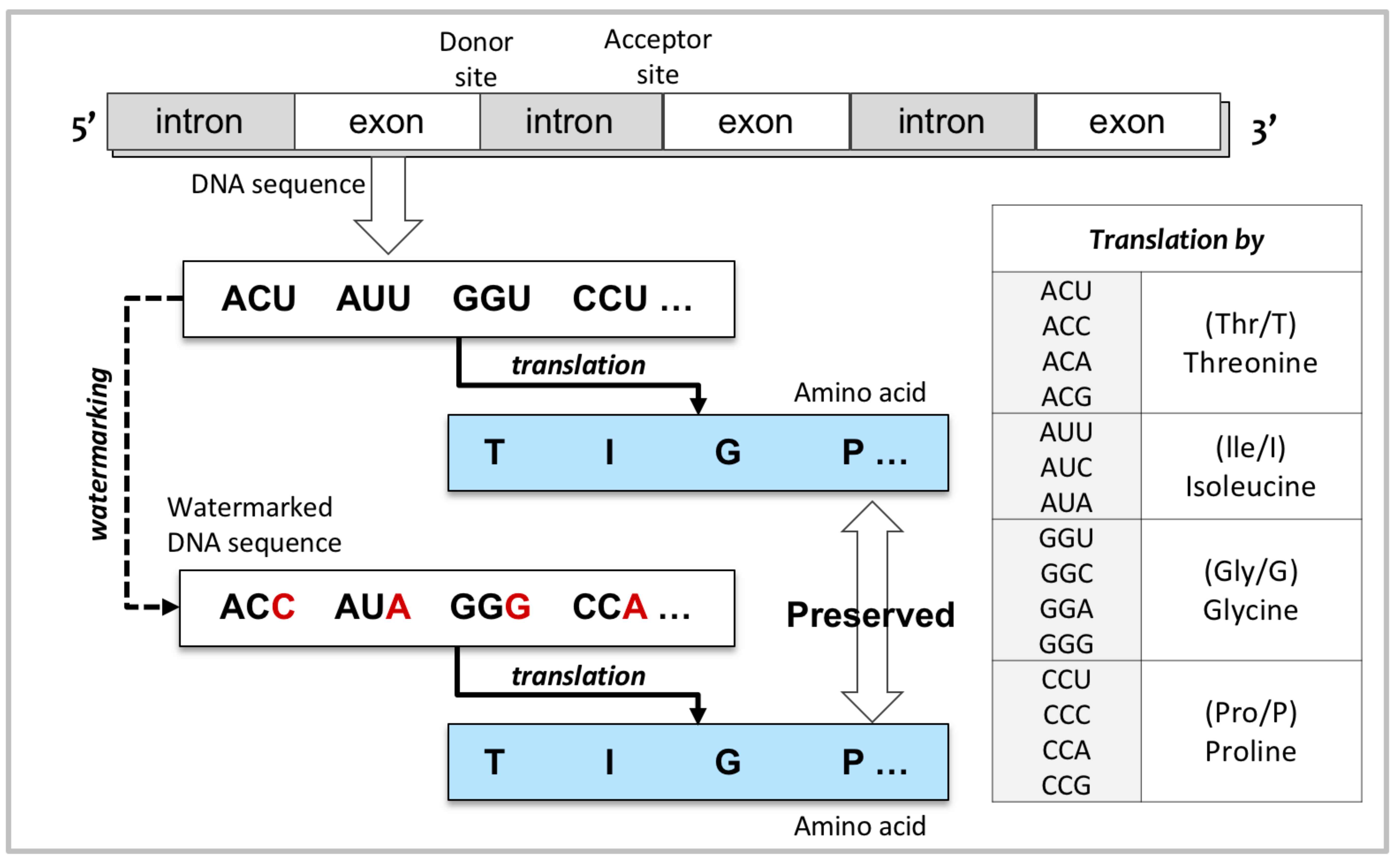}
        \caption{\textcolor{\revisionOne}{DNA hiding scheme using synonymous codons. A watermark is a scheme \textcolor{\ediThree}{used} to deter unauthorized dissemination by marking hidden symbols or texts.} 
For the conservation of amino acids, DNA watermarking can be changed to one of the synonymous codons.}
        \label{fig:watermarked}
    \end{minipage}
    \begin{minipage}{.45\textwidth}
    	\centering
        \includegraphics[width=0.920\linewidth]{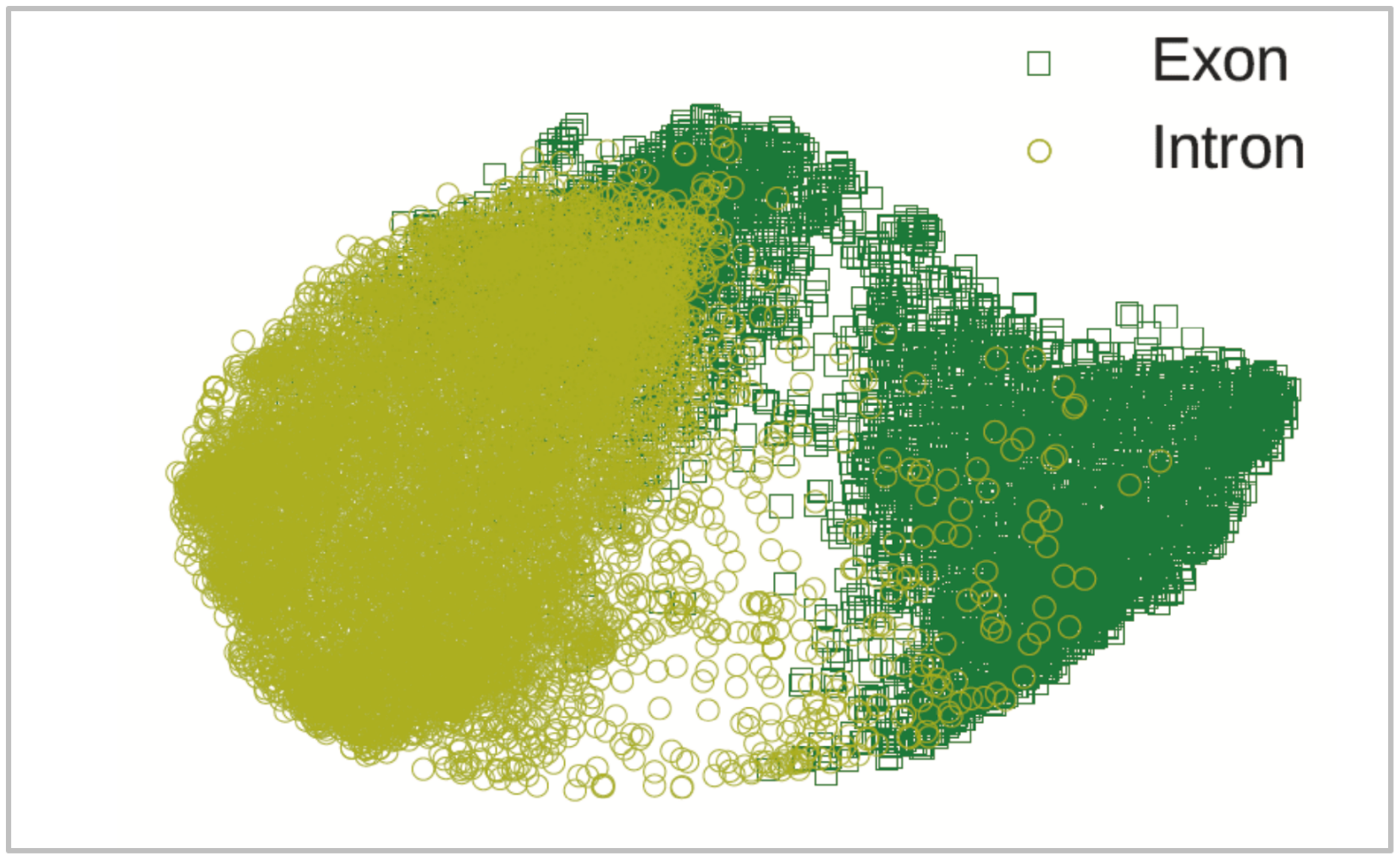}
        \caption{Learned representation of DNA sequences. The learned representations for each coding and non-coding region projected into a two-dimensional (2-D) space using t-SNE~\cite{maaten2008visualizing}. The representation is based on sequence-to-sequence learning using an autoencoder and stacked RNNs.} 
    \end{minipage}
    \label{fig:tsne}
\end{figure}


We use the standard terminology of information hiding{~\cite{anderson1996information}} to \textcolor{\editage}{provide} a brief explanation of \textcolor{\ediThree}{the} \textcolor{\revisionTwo}{related background}.
\textcolor{\revisionfour}{
\textcolor{\ho}{For example}, two hypothetical parties, \textcolor{\ho}{(i.e., a sender and a receiver)} wish to exchange genetically modified organisms (GMOs) protected by patents. \textcolor{\ho}{A} third party detects watermark sequence from the GMOs for unauthorized use. Both the sender and receiver use the random \textcolor{\ho}{oracle{~\cite{canetti2004random}}} model, which posits existing steganography schemes, to embed their watermark message, and the third party uses our proposed model to detect the watermarked signal.}
A random oracle model posits the randomly chosen function $H$, which can be evaluated only by \textcolor{\ho}{querying the} oracle that returns $H(m)$ given input $m$. 

\subsection{Notations}
The notations used in this paper are as follows: $\textbf{D} = \{D_1, \cdots, D_n\}$ is a set of DNA sequences of $n$ species; ${\hat{\textbf{D}}} = \{\hat{D}_1, \cdots, \hat{D}_n\}$ is a set of DNA sequences of $n$ species and the hidden messages are embedded for some species $\hat{D}_i$; $m \in \{\texttt{A,C,G,T}\}^{\textcolor{\fix}\ell}$ is the input sequence where $\ell$ is the length of the input sequence; $\hat{m} \in \{\texttt{A,C,G,T}\}^{\ell}$ is the encrypted value of $m$ where $\ell$ is the length of the encrypted sequence; $E$ is an encryption function, which takes input $m$ and returns the encrypted sequence $E(m) \rightarrow \hat{m}$; $\mathbf{M}_{D_i}$ is a trained model that takes target species $D_i$ as training input; $\overline{y}$ is an averaged output score $y$; $\hat{y}$ is a probability output given by the trained model $\mathbf{M}_{D_i}(\hat{m}) \rightarrow \textcolor{\revisionTwo}{\hat{y}}$ given input $\hat{m}$, where $\hat{m} \in \hat{D}_i$; $\mathcal{A}$ is a probabilistic polynomial-time adversary. The \textcolor{\ho}{adversary{\cite{bellare1993random}}} is an attacker that queries messages to the oracle model; $\epsilon$ is the standard deviation value of \textcolor{\revisionTwo}{score} $y$.

\subsection{Hiding Messages}\label{hiding_space}

The \textcolor{\revisionTwo}{hiding} positions of \textcolor{\ediThree}{a} DNA sequence segment are \textcolor{\ediThree}{limited} compared to \textcolor{\editageTwo}{those of the} covert channel \textcolor{\editageTwo}{because} the sequences are carried over after the translation and transcription processes in the exon region. 
For example, assume that \texttt{ACGGTTCCAATGC} is a reference sequence, and 01001100 is the message to be hidden. The reference sequence is then translated according to any coding schemes. In this example, we apply the DNA-crypt coding scheme{~\cite{heider2007dna}}, which converts the DNA sequence to binary replacing \textcolor{\ediThree}{\texttt{A} with 00, \texttt{C} with 01, \texttt{G} with 10, and \texttt{T} with 11}. The reference sequence is then translated to 00011010111101010000111001 and divided into key bits \textcolor{\editageTwo}{that are} defined by the sender and receiver. Assume that the \textcolor{\revisionTwo}{length of the key is 3}, the reference sequence can be expressed as  000, 110, 101, 111, 010, 100, 001, 110, 01, and \textcolor{\editageTwo}{the message is} concealed at the \textcolor{\revisionThree}{first} position. The DNA sequence with the concealed messages are then represented as 0000, 1110, 0101, 0111, 1010, 1100, 0001, 0110, 01. Finally, the sender transmits the transformed DNA sequence of \texttt{AATGCCCTGGTAACCG}. The recipient can extract the hidden message using \textcolor{\revisionTwo}{the} \textcolor{\ediThree}{pre-}defined key.

\subsection{Determination of Message-Hiding Regions}
Genomic sequence regions \textcolor{\ediThree}{(i.e., exons and introns)} are utilized depending on whether the task is data storage or transport.
Intron regions are \textcolor{\ediThree}{suitable} for transportation \textcolor{\ediThree}{since they are not transcribed and are} removed by splicing{~\cite{keren2010alternative, lockhart2000genomics}} during transcription.
This \textcolor{\ediThree}{property of introns} provides \textcolor{\ediThree}{large sequence} space for \textcolor{\ediThree}{concealing} data, creating potential covert channels.
\textcolor{\editageTwo}{In} contrast, \textcolor{\ediThree}{data storage (watermarking) requires data to} be resistant to degradation or truncation.
Exons are a suitable \textcolor{\editage}{candidate} for storage  \textcolor{\editage}{because} \textcolor{\ediThree}{underlying DNA} sequence \textcolor{\ediThree}{is conserved} after the translation and transcription processes{~\cite{shimanovsky2002hiding}}.
These two components of internal structure \textcolor{\ediThree}{components} in eukaryote genes are involved \textcolor{\editage}{in} DNA steganography as \textcolor{\editageTwo}{the} payload (watermarking) or carrier (covert channels). \figurename~\ref{fig:tsne} shows the learned representations of introns and exons which are calculated by softmax function. The softmax function reduces the outputs of intron and exons to range between 0 and 1. 
The \textcolor{\editageTwo}{2D} projection position of introns and exons will change if hidden messages are embedded without considering shared patterns between \textcolor{\ediThree}{the genetic components (e.g., complementary pair rules)}. Thus, the construction of a classification model to enable a clear separation axis of these shared patterns \textcolor{\editageTwo}{is an} important factor in the detection of hidden messages.

\section{Methods}
\begin{figure*}[t!]
\includegraphics[width=0.90\textwidth]{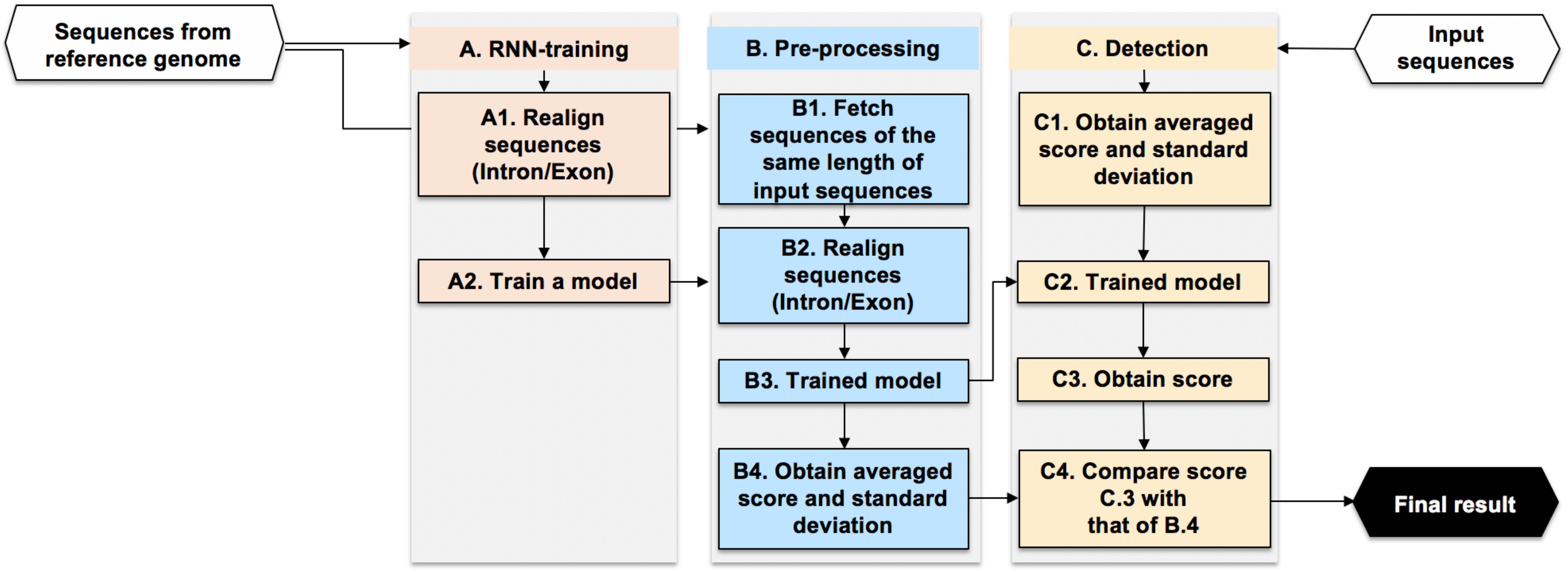}
\caption{Flowchart of proposed DNA steganalysis pipeline.}
\label{fig:flowchart}
\end{figure*}

Our proposed method uses RNNs{~\cite{schmidhuber2015deep} to detect hidden messages in DNA. \figurename~\ref{fig:flowchart} shows our proposed \textcolor{\ediThree}{steganalysis pipeline}. The pipeline \textcolor{\ediThree}{comprises} of training and detection phases. In the model training phase, the model learns \textcolor{\revisionTwo}{the} distribution of unmodified genome sequences that \textcolor{\editageTwo}{distinguishes between} introns and exons (see Section~\ref{proposed_model} for the model architecture). In the detection phase, we obtain a prediction score exhibiting the distribution of introns and exons. By exploiting the obtained prediction score, we formulate a detection principle.
The details of the detection principle are described in Section~\ref{proposed_steg}.

\subsection{\textcolor{\ho}{\textcolor{\ediThree}{Proposed DNA} Steganalysis Principle}}\label{proposed_steg}
The security of the random oracle is based on an \textcolor{\ho}{\emph{experiment} $E$} involving an adversary $\mathcal{A}$, \textcolor{\editageTwo}{as well as} $\mathcal{A}$'s indistinguishability of the encryption.
Assume that we have the random oracle that acts like a current steganography scheme \textcolor{\ho}{$S$} with only a negligible success probability.
The experiment \textcolor{\ho}{$E$} can be defined for any encryption scheme $S$ over message space $\textbf{D}$ and for adversary $\mathcal{A}$.
We describe the proposed method to detect hidden messages using the random oracle. For the \textcolor{\ho}{$E$}, the random oracle chooses a random steganography scheme \textcolor{\ho}{$S$}. 
    \textcolor{\ho}{Scheme $S$} modifies or extends \textcolor{\editage}{the process of} mapping a sequence \textcolor{\editageTwo}{with} length $n$ input to a sequence \textcolor{\editageTwo}{with} length $\textcolor{\fix}{\ell}$ \textcolor{\editageTwo}{with a} random sequence as the output. 
\textcolor{\editageTwo}{The} process of mapping sequences can be \textcolor{\editageTwo}{considered} as a table \textcolor{\editageTwo}{that} indicates for each possible input $m$ the corresponding output value $\hat{m}$.
With chosen \textcolor{\ho}{scheme $S$}, $\mathcal{A}$ chooses a pair of sequences {$m_0,m_1 \in D_i$}. The random oracle which posits the \textcolor{\ho}{scheme $S$} selects a bit \textcolor{\revisionThree}{ $b \in \{0,1\}$} and sends encrypted message $ S(m_b) \textcolor{\revisionTwo}{\rightarrow} \hat{m}$ to the adversary. The adversary outputs a bit $b'$. Finally, the output of the \textcolor{\ho}{$E$} is defined as 1 if $b' = b$, and 0 otherwise. $\mathcal{A}$ succeeds in the \textcolor{\ho}{$E$} in the case of distinguishing $m_b$. \textcolor{\ediThree}{Our methodology using $E$ is described} as follows:
\begin{enumerate} [leftmargin=*,label=(\roman*), labelindent=1.5mm,labelsep=1.3mm]
    \item We construct $\textup{M}_{D_i}$ (\figurename~\ref{fig:flowchart}-A)  that runs on a random oracle where selected species $D_i \in \textbf{D}$.
     Note that a model $M$ can be based on any classification model, but the key to select a model is to reduce the standard deviation. Our proposed model $M$ is described in Section~\ref{proposed_model}.

    \item \textcolor{\revisionThree}{$\mathcal{A}$ computes $y$ (\figurename~\ref{fig:flowchart}-B4) using $M_{D_i}(m)$ given $m \in D_i$.}
    
     \item \textcolor{\revisionThree}{$\mathcal{A}$ computes the standard deviation $\epsilon$ of $y$ (\figurename~\ref{fig:flowchart}-B)}.

    \item $\mathcal{A}$ computes $\hat{y}$ \textcolor{\revisionTwo}{(\figurename~\ref{fig:flowchart}-C3)} using $M_{D_i}(\hat{m})$ given $\hat{m} \in \hat{D}_i$.
    
    \item $\hat{m}$ is successfully detected \textcolor{\revisionTwo}{(\figurename~\ref{fig:flowchart}-C4)} if
    \begin{equation}\label{equation2}
        |  \textcolor{\revisionThree}{\overline{y}} - \hat{y} | > \epsilon. 
    \end{equation}
\end{enumerate}
\begin{figure*}[t!]
\centering
\includegraphics[width=0.95\textwidth]{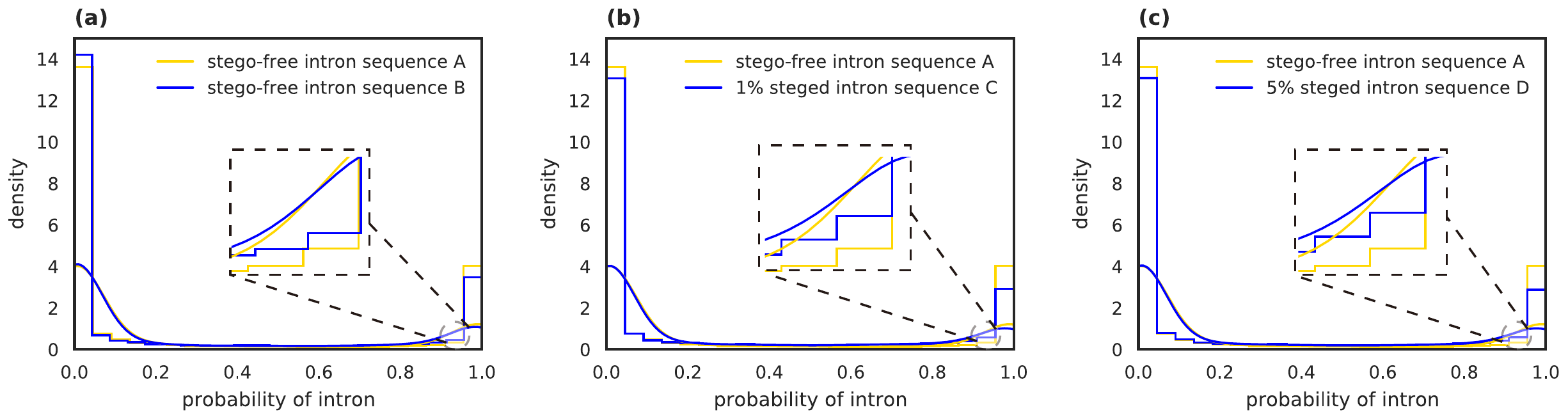}
\caption{Final score of intron/exon sequence obtained from the softmax of the neural network (best viewed in color).
(a) kernel density differences between two stego-free intron sequences
(b) kernel density differences between stego-free and 1\% perturbed steged intron sequences.
(c) kernel density differences between stego-free and 5\% perturbed steged intron sequences.
}
\label{fig:probability}
\end{figure*}
\noindent
\textcolor{\revisionTwo}{This gives two independent scores $y$ and $\hat{y}$ from $M_{D_i}$. The score $y$ will have the same range of \textcolor{\editageTwo}{the} unmodified genome sequences whereas the score $\hat{y}$ will have a different range of modified genome sequences.
If the score difference between y and $\hat{y}$ is larger than the standard deviation of \textcolor{\editageTwo}{the} unmodified genome sequence distribution, it \textcolor{\editageTwo}{may} be that the sequence has been forcibly changed.
\figurename~\ref{fig:probability} shows the histogram of the final score of $y$ and $\hat{y}$ returned from softmax of the neural network. If the message is hidden, we can see that the final score from softmax of the neural network differs \textcolor{\editageTwo}{over} the range $\overline{y} \pm \epsilon$.}
\textcolor{\editageTwo}{From Eq.~(\ref{equation2}) below,} we show that detection is possible using information theoretical proof based on entropy \textcolor{\ho}{$H$~({Ref.~\cite{blahut1987principles}})}.

\begin{lemma}
A DNA steganography scheme is not secure if \textup{$H(\textbf{D}) \neq H(\hat{\textbf{D}}|\textbf{D})$}.
\end{lemma}
\begin{proof}
The mutual joint entropy $H(\textbf{D},\hat{\textbf{D}}) = H(\textbf{D}) + H(\hat{\textbf{D}}|\textbf{D})$ is the union of both entropies for distribution $\textbf{D}$ and $\hat{\textbf{D}}$.
According to Gallager at el{~\cite{gallager1968information}}, the mutual information of $I(\textbf{D};\hat{\textbf{D}})$ is \textcolor{\editageTwo}{given as} $I(\textbf{D};\hat{\textbf{D}}) = H(\textbf{D}) - H(\textbf{D}|\hat{\textbf{D}})$. It is symmetric in $\textbf{D}$ and ${\hat{\textbf{D}}}$ such that $I(\textbf{D};\hat{\textbf{D}}) = I(\hat{\textbf{D}};\textbf{D})$, and always non-negative.
The conditional entropy between two distribution is 0 if and only if the distributions are equal. Thus, the mutual information \textcolor{\editage}{must} be zero to define secure DNA steganography schemes:
\begin{equation}\label{equation3}
I(\textbf{C}; (\textbf{D},\hat{\textbf{D}})) = H(\textbf{C}) - H(\textbf{C} | (\textbf{D},\hat{\textbf{D}}))  = 0.
\end{equation}
\noindent
where \textbf{C} is message hiding space and it follows \textcolor{\editage}{that}:
\begin{equation}
 H(\textbf{C}) = H( \textbf{C} | (\textbf{D},\hat{\textbf{D}})).
\end{equation}
Eq.~(\ref{equation3}) \textcolor{\editageTwo}{indicates} that the amount of entropy $H(\textbf{C})$ must not be decreased \textcolor{\editage}{based on} the knowledge of $\textbf{D}$ and $\hat{\textbf{D}}$.
It follows that the secure steganography scheme is obtained if and only if:
\begin{equation}
    \forall_i \in \mathbb{N}, m_i \in \textbf{D}, \hat{m_i} \in {\hat{\textbf{D}}} : m_i = \hat{m_i}. \nonumber
\end{equation}
\textcolor{\fix}{
Note that \textcolor{\editageTwo}{for} $m_i=\hat{m_i}$ \textcolor{\editageTwo}{it} is \textcolor{\editageTwo}{not possible to} distinguish \textcolor{\editageTwo}{between the} original sequence and the stego sequence.}
Considering that the representations of $\hat{m}$ are limited to  \{\texttt{A,C,G,T}\}, \textcolor{\editageTwo}{it is nearly impossible to satisfy} the condition \textcolor{\editage}{because} current steganography schemes are all based on the assumption of addition or substitution.
\textcolor{\editage}{Because} $\textbf{C}$ is independent of $\textbf{D}$, the amount of information will increase over distribution $\textbf{D}$ if hidden messages are inserted over distribution ${\hat{\textbf{D}}}$.
We can conclude \textcolor{\editage}{that} the schemes are not secure under condition $H(\textbf{C}) > H(\textbf{C} | (\textbf{D},{\hat{\textbf{D})}})$. 
\end{proof}
\noindent

\begin{figure*}[t!]
\includegraphics[width=0.85\textwidth]{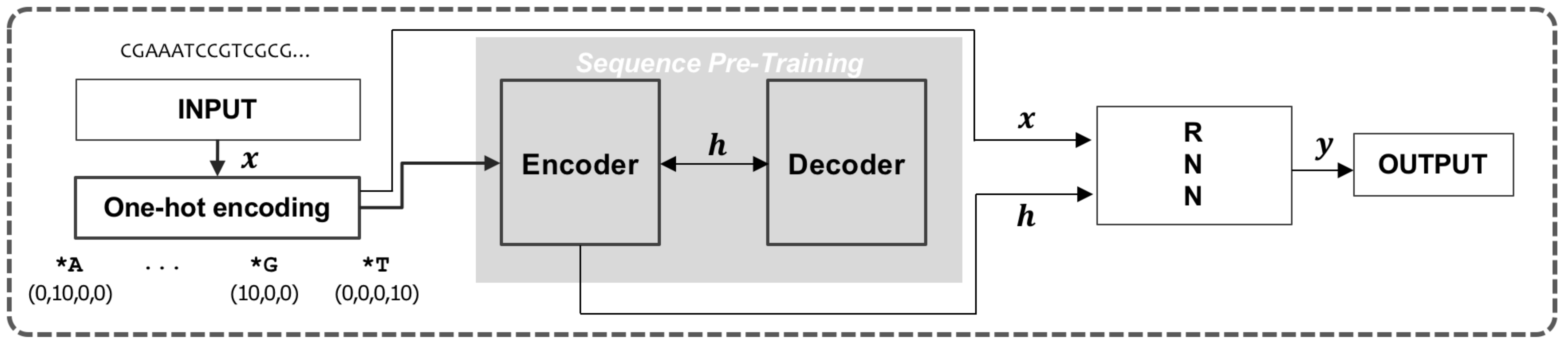}
\centering
\caption{Overview of proposed RNN methodology.}
\label{fig:model_structure}
\end{figure*}

\subsection{Proposed Steganalysis \textcolor{\revisionThree}{RNN} Model}\label{proposed_model}
The proposed model is based on sequence-to-sequence learning using an autoencoder and stacked RNNs{~\cite{peterson2014common}}, \textcolor{\editageTwo}{where} the model training \textcolor{\ediThree}{consists of} two main steps: 1) unsupervised pre-training of sequence-to-sequence autoencoder for modeling an overcomplete case, and 2) supervised fine-tuning of stacked RNNs for modeling patterns between canonical and non-canonical splice sites (see \figurename~\ref{fig:model_structure}).
In the proposed model, we use a set of DNA sequences \textcolor{\ediThree}{labeled as} introns and exons. These sequences are converted into a binary vector by orthogonal encoding{~\cite{baldi2001bioinformatics}}.
It employs $n_c$-bit one-hot encoding. For $n_c= 4$, $\{ \texttt{A,C,T,G} \}$ is encoded by 
\begin{equation}\label{input_representation}
   \langle[1,0,0,0], [0,1,0,0], [0,0,1,0], [0,0,0,1]\rangle.
\end{equation}
For example, the sequence $\texttt{ATTT}$ is encoded into a $4 \times 4$ dimensional binary vector $\langle[1,0,0,0],[0,0,0,1],[0,0,0,1],[0,0,0,1]\rangle$. \textcolor{\revisionTwo}{The encoded sequence} is a tuple of \textcolor{\editageTwo}{a} four-dimensional \textcolor{\editageTwo}{(4D)} dense vector, \textcolor{\editageTwo}{and} is connected to the first layer of an autoencoder, which is \textcolor{\editage}{used} for the unsupervised pre-training of sequence-to-sequence learning. An autoencoder is an artificial neural network \textcolor{\editageTwo}{(ANN) that is} used \textcolor{\editage}{to learn} meaningful encoding for a set of data in \textcolor{\editage}{a case involving} unsupervised learning. \textcolor{\editage}{An autoencoder} consists of two components, \textcolor{\editageTwo}{namely} an encoder and decoder. 

\textcolor{\revisionTwo}{The encoder RNN encodes $\mathbf{x}$ to the representation of sequence features $\mathbf{h}$}, and the decoder RNN decodes $\mathbf{h}$ to the reconstructed $\hat{\mathbf{x}}$; \textcolor{\editage}{thus} minimizing the reconstruction errors of
$\mathcal{L}(\mathbf{x}, \mathbf{\hat{x}}) = \lVert \mathbf{x} - \mathbf{\hat{x}} \rVert^{2}$,
\textcolor{\revisionTwo}{
where $\mathbf{x}$ is one-hot encoded input.}
Through unsupervised learning of the encoder-decoder model{~\cite{srivastava2015unsupervised}}, we \textcolor{\editage}{obtain} representations of inherent features $\mathbf{h}$, which are directly connected to the second activation layer. The second layer is RNNs layer \textcolor{\editage}{used} to construct the model. \textcolor{\editage}{The model in turn is used} to determine patterns between canonical and non-canonical splice signals.
We then obtain the tuple of fine-tunned
    $\mathbf{h} = <\mathbf{h_1}, \cdots , \mathbf{h_d}>$,
\textcolor{\revisionTwo}{where $\textbf{h}$ is the representation of sequence features learned by features}, which is a representation of introns and exons in hidden layers, and $\mathbf{d}$ \textcolor{\fix}{is the} dimension of a vector.

The features $\textbf{h}$ learned from the autoencoder are connected to the second stacked RNN layer, which \textcolor{\editage}{consists of} our proposed architecture for outputting \textcolor{\editageTwo}{a} classification score for the given sequence $D_i \in \textbf{D}$. For the fully connected output layer, we use the sigmoid function as the activation. The activation score is given by
 $   \text{Pr}(\textcolor{\revisionThree}{y}=i | \mathbf{h}) = \frac{1/(1+\mathrm{exp}(-\mathbf{w}_{i}^{T}\mathbf{h}))}{\sum_{k=0}^{1} 1/(1+\mathrm{exp}(-\mathbf{w}_{k}^{T}\mathbf{h}))}$,
\noindent
where $y$ is the label \textcolor{\editage}{that indicates} whether the given region \textcolor{\editage}{contains} \textcolor{\revisionThree}{introns ($y=1$) or exons ($y=0$)}.
\textcolor{\editageTwo}{For our training model,} we use a recently proposed optimizer of multi-class logarithmic loss function Adam{~\cite{kingma2014adam}}.
The objective function $\mathcal{L}(\mathbf{w})$ that \textcolor{\editage}{must} be minimized is \textcolor{\editage}{defined} as follows:
\begin{equation}\label{loss_function}
    \mathcal{L}(\mathbf{w}) = -\frac{1}{N} \sum_{n=1}^{N} (y_{i}\textrm{log}(p_{i}) + (1-y_{i})\textrm{log}(1-p_{i}))
\end{equation}
where $N$ is the mini-batch size.
A model $\mathbf{M}_{D_i}$ has a possible score of $p_i$ for one species, where $p_i$ is the score of given non perturbed sequences.  

\section{Results}
\begin{figure*}[t!]
\includegraphics[width=0.85\textwidth]{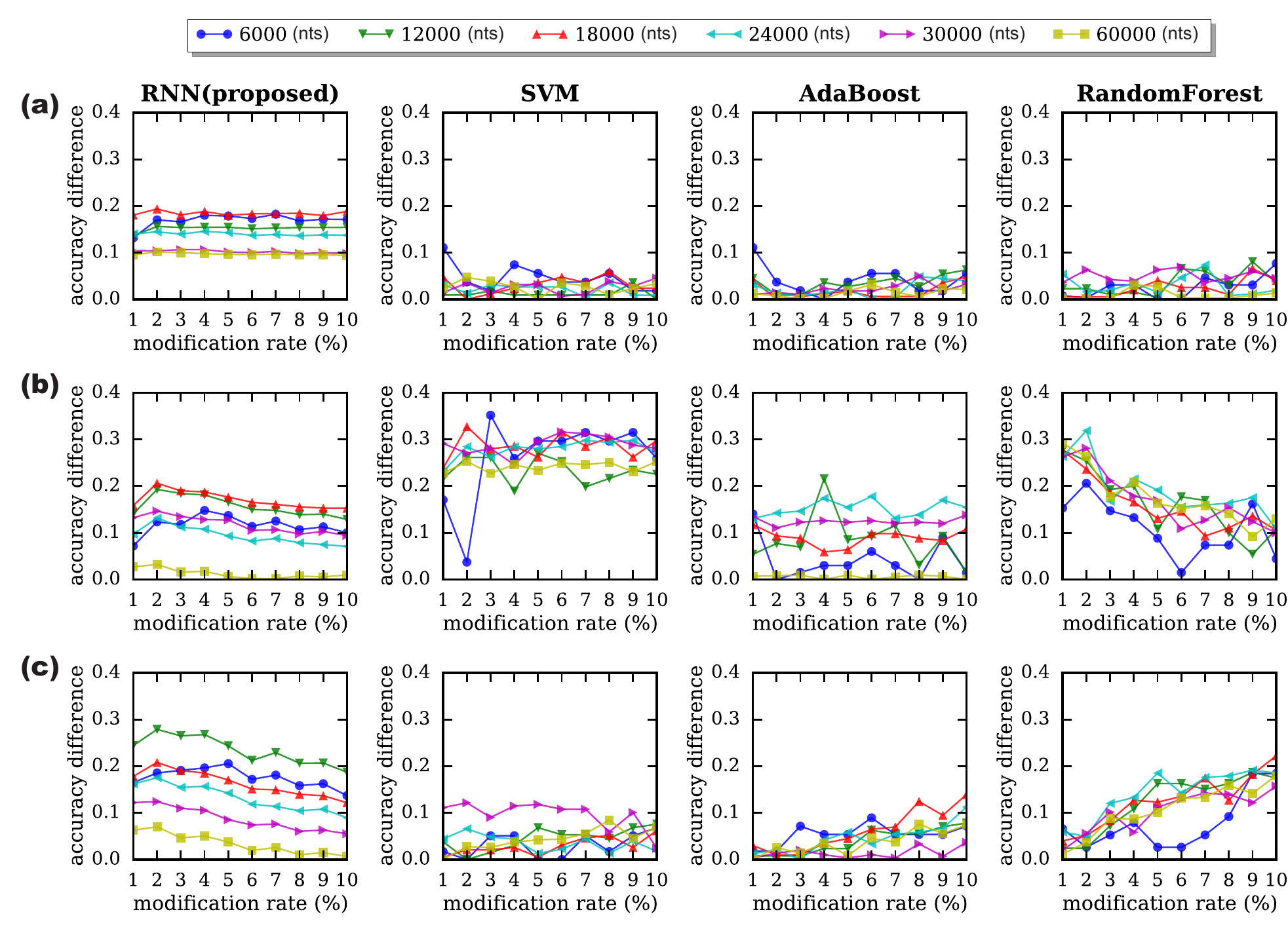}
\centering
\caption{Comparison of learning algorithms with random hiding algorithms (best viewed in color). (a) differences in accuracy for intron region (b) differences in accuracy for exon region (c) difference in accuracy for both region. [The performances of four supervised learning algorithms when detecting hidden messages are shown for six variable lengths of nucleotides (nts).]}

\label{fig:compare_by_samples}
\end{figure*}

\subsection{Dataset}
We simulated our approach \textcolor{\editageTwo}{using} the \textcolor{\revisionTwo}{Ensembl} human genome dataset and human UCSC-hg38 dataset{~\cite{kent2002human}}, which include sequences from 24 human chromosomes (22 autosomes and 2 sex chromosomes).
The Ensembl human genome dataset has a two-class classification (coding, and non-coding) and the UCSC-hg38 dataset has a three-class classification (donor, acceptor, and non-site).

\subsection{Input Representation}
The machine learning approach typically employs \textcolor{\editageTwo}{a} numerical representation of the input for downstream processing. Orthogonal encoding, such as one-hot coding{~\cite{baldi2001bioinformatics}}, is widely used to convert DNA sequences into a numerical format. It employs $n_c$-bit one-hot encoding. For $n_c = 4$,  $\{ \texttt{A,C,T,G} \}$ is encoded as described in Eq.~(\ref{input_representation}).
According to Lee et al.{~\cite{lee2015dna}}, the vanilla one-hot encoding scheme tends to limit generalization because of the sparsity of its encoding (75\% of the elements are zero). Thus, our approach encodes nucleotides into a \textcolor{\editageTwo}{4D} dense vector \textcolor{\editage}{that follows} the direct architecture of a normal neural network layer{~\cite{chollet2015keras}}, which is trained by the gradient decent method.

\subsection{Model Training}
\textcolor{\revisionTwo}{The proposed RNN-based approach uses unsupervised training for the autoencoder and supervised training for the fine-tuning. The first layer of unsupervised training uses 4 input units, 60 hidden RNNs units with 50 epochs and 4 output units \textcolor{\editageTwo}{that are} connected to the second layer. The second layer of supervised training uses 4 input units \textcolor{\revisionTwo}{that are} connected to stacked LSTM layers with full version including forget gates and peephole connections.} The 4 input layers are used for 60 hidden  units with 100 epochs, and the 4 output units are a fully connected output layer containing $K$ units for $K$-class prediction. 

\textcolor{\fix}{In our experiment, we used $K=2$ to classify sequences (coding or non-coding)}. 
\textcolor{\revisionfour}
{For the fully connected output layer, we used the softmax function to classify sequences and the sigmoid function to classify sites for the activation.}
\textcolor{\editageTwo}{For our training model,} we used a recently proposed optimizer of multi-class logarithmic loss function Adam{~\cite{kingma2014adam}}. The objective function $\mathcal{L}(\mathbf{w})$ that has to be minimized is as described in \textcolor{\revisionThree}{Eq~(\ref{loss_function})}. We used a batch size of 100 and followed the batch normalization{~\cite{ioffe2015batch}}. We initialized weights according to a uniform distribution as directed by Glorot and Bengio{~\cite{glorot2010understanding}}. The training time was approximately 46 hours and the running time was less than 1 second (Ubuntu 14.04 on 3.5GHz i7-5930K and 12GB  Titan X).

\begin{figure*}[t!]
\centering
\includegraphics[width=0.875\textwidth]{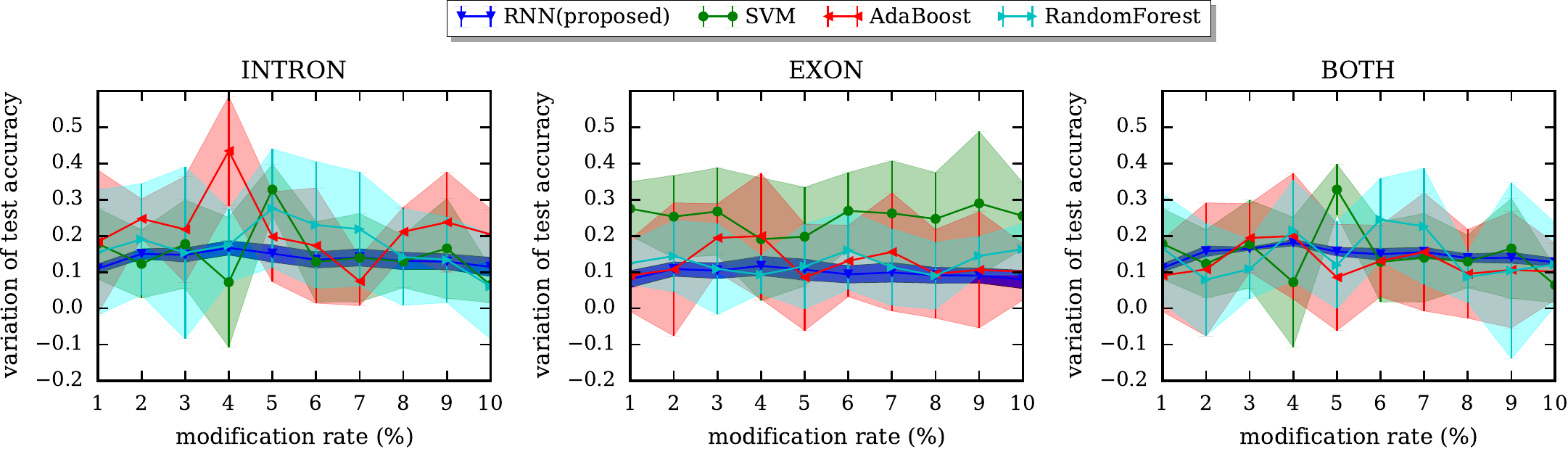}
\caption{Comparison of learning algorithms in terms of robustness (best viewed in color). Mean and variance of accuracy are measured for the fixed \textcolor{\fix}{DNA sequence length of 6000} for 500 cases by changing one percent of the hidden message. The shaded line represents the standard deviation of the inference accuracy.}
\label{fig:errorbars}
\end{figure*}

\textcolor{\revisionTwo}{
\subsection{Evaluation Procedure}
For evaluation \textcolor{\ediThree}{of} performance, we used the score obtained from the softmax of the neural network. We exploited the state-of-the-art algorithm{~\cite{mitras2013proposed}} to embed hidden messages for the message hiding. We randomly selected DNA sequences from the validation set using the \textcolor{\revisionTwo}{Ensembl} human genome dataset. We obtained the score of the stego-free sequence from the validation set. In the next step, we embedded hidden messages to a selected DNA sequence from the validation set, and \textcolor{\editageTwo}{we} obtained the score. Using the score distribution of the stego-free and steged sequences, we evaluated the \textcolor{\revisionTwo}{different scores for} the range $\overline{y} \pm \epsilon$.
The output from softmax of the neural network \textcolor{\editageTwo}{is expected to} have a similar score distribution \textcolor{\editageTwo}{as} the unmodified genome sequences. However, the score distribution changes if messages are embedded. As shown in \figurename~\ref{fig:probability}(b) and \figurename~\ref{fig:probability}(c), modified sequences are distinguishable using our RNNs model. 
}

\subsection{Performance Comparison}
We evaluated the performance of our proposed method based on four supervised learning algorithms (RNNs, SVM, random forests, and adaptive boosting) \textcolor{\editageTwo}{to detect} hidden messages. 
For the performance metric, we used \textcolor{\editageTwo}{the differences in} accuracy.
\footnote{$\textrm{Accuracy} = (TP+TN) / (TP+TN+FP+FN)$, where $TP$, $FP$, $FN$, and $TN$ represent the numbers of true positives, false positives, false negatives, and true negatives, respectively.}
\textcolor{\ediThree}{Using the prediction performance data}, we evaluated learning algorithms with respect to \textcolor{\editage}{the following} three regions; introns dedicated, exons dedicated, and both regions \textcolor{\editage}{together}.

For each algorithm, we generated simulated data for different \textcolor{\fix}{lengths of DNA sequences (6000, 12000, 18000, 24000, 30000, and 60000)} using the UCSC-hg38 dataset{~\cite{kent2002human}}.
\textcolor{\revisionfour}{We also randomly selected 1000 cases for the fixed DNA sequence length for the modification rate 1 to 10\%.} 
\textcolor{\editageTwo}{Using} selected DNA sequences, we obtained the average prediction accuracy of different numbers of samples against non-perturbed DNA sequences for 1000 randomly selected cases. 
In the next step, we obtain the prediction accuracy for the modified \textcolor{\editageTwo}{data} generated according to the hiding algorithms. 
Using \textcolor{\editageTwo}{the} averaged prediction accuracy \textcolor{\editageTwo}{for both the} perturbed and non-perturbed \textcolor{\editageTwo}{cases}, we evaluated the differences between the prediction accuracy rates for \textcolor{\editageTwo}{varying} different numbers of samples. We carried \textcolor{\fix}{out} five-fold cross-validation to obtain the mean/variance of \textcolor{\editageTwo}{the differences in} accuracy.

\begin{table}[t!]
\tbl{Detection performance of sequence alignment and denoising tools.}
{\resizebox{.80\columnwidth}{!}
{\begin{tabular}{l|ccc}
\hline
 &\thinspace \textbf{Both Region (\%)}\thinspace &\thinspace \thickspace \textbf{Intron Region (\%)} \thinspace & \thickspace \thinspace \textbf{Exon Region (\%)}\thinspace \\\hline
 
\textbf{RNN \ (proposed)} & \textbf{99.93} &  \textbf{99.96} &  \textbf{99.94} \\
  BLAST \cite{altschul1990basic} \thickspace & 84.00 & 85.00 & 85.00 \\
  Coral \cite{salmela2010correction} & 0.00 & 0.00 & 0.00 \\
  Lighter \cite{song2014lighter} & 0.00 & 0.00 & 0.00 \\
\hline
\end{tabular}}\label{table:other_tools}
}
\end{table}

\figurename~\ref{fig:compare_by_samples} shows an experiment for each algorithm using six \textcolor{\fix}{variable DNA sequence lengths.} Each algorithm was compared to three different regions \textcolor{\editage}{based on the} six \textcolor{\fix}{variable DNA sequence lengths}. The experiments were \textcolor{\editage}{conducted by} changing from one to then percent of the hidden message. SVM showed good detection performance in the exon region, but showed inferior performance in the intron \textcolor{\ediThree}{as well as} both regions \textcolor{\ediThree}{category}. In the case of adaptive boosting, the detection performance was similar in both \textcolor{\editage}{regions} and \textcolor{\editage}{in} intron \textcolor{\ediThree}{only categorie}, but performed poorly in exon regions. In the case of the random forest, \textcolor{\editageTwo}{the cases with the} exon and both regions showed good performance except for some modification rates. In the intron regions, the detection performance was similar to \textcolor{\editage}{that of} other learning algorithms. \textcolor{\ediThree}{Notably}, our proposed methodology based on RNNs outperformed \textcolor{\ediThree}{all of the existing} hidden messages detection \textcolor{\ediThree}{algorithms for} all \textcolor{\ediThree}{genomic} regions \textcolor{\ediThree}{evaluated}.

\textcolor{\editage}{In addition}, we examined our proposed methodology based on denoising methods using Coral{~\cite{salmela2010correction}} and Lighter{~\cite{song2014lighter}}. The UCSC-hg38 dataset was used to preserve local base structures and perturbed data samples were used as random noise. As shown in Table~\ref{table:other_tools}, the results showed that both Coral and Lighter missed detection for all modification rates in all regions. \textcolor{\editageTwo}{In addition}, the sequence alignment \textcolor{\editageTwo}{method} performed poorly. The results suggest \textcolor{\editageTwo}{that there is a} 15 to 16\% chance that hidden messages may not be detected in all three regions.

To validate the learning algorithms \textcolor{\editage}{with respect to} robustness, we tested \textcolor{\editage}{them} with a fixed \textcolor{\fix}{DNA sequence length of 6000} with 500 cases for each modification rate to measure the mean and variance of the test accuracy.
\figurename~\ref{fig:errorbars} shows how the performance measures (mean and variance of accuracy differences) change for modification rates \textcolor{\editageTwo}{ranging} from 1 to 10 in the intron, exon, and both regions \textcolor{\ediThree}{categories}.
\textcolor{\editageTwo}{The} plotted entr\textcolor{\editageTwo}{ies represents the} the averaged mean over the 500 cases, and shade lines show the average of the variances over \textcolor{\editageTwo}{the} 500 cases.
The results \textcolor{\editageTwo}{indicate} that hidden messages may not be detected if the prediction difference is less than \textcolor{\editageTwo}{the} variance.
The overall analysis \textcolor{\fix}{with respect to the} robustness \textcolor{\editage}{showed that the} learning algorithms of SVM, random forests and adaptive boosting performed poorly.

\section{Discussion}

\textcolor{\editageTwo}{The development of} next-generation sequencing has reduced the price of personal genomics{~\cite{schuster2008next}}, and the discovery of the \textcolor{\ho}{CRIPSPR-Cas9 gene has provided unprecedented control over genomes of many species{~\cite{hsu2014development}}.} \textcolor{\editageTwo}{While the} technology is yet to \textcolor{\ediThree}{be} \textcolor{\editageTwo}{applied to simulations involving} artificial DNA, human DNA sequences may become an \textcolor{\editageTwo}{area in} \textcolor{\editage}{which we can} apply DNA watermarking.
Our experiments using the real UCSC-hg38 human genome implicitly \textcolor{\editage}{consider} that unknown relevant sequences are also detectable because of the characteristics of similar patterns in non-canonical splice sites. The number of donors with \texttt{GT} pairs and acceptors with \texttt{AG} pairs \textcolor{\editage}{were found to be} 86.32\% and 84.63\%, respectively{~\cite{lee2015boosted}}.
Existing steganography techniques modify several nucleotides. \textcolor{\fix}{Considering few single nucleotide modifications, we can transform DNA steganography to the variant calling problem. In this \textcolor{\editageTwo}{regard}, we believe \textcolor{\editageTwo}{that} our methodology can be extended to the field of variant calling.}

Although \textcolor{\editageTwo}{there are} many advantages \textcolor{\editage}{to} using machine learning techniques \textcolor{\editageTwo}{to detect} hidden messages{~\cite{lyu2004steganalysis,erfani2016high,min2017deep}}, \textcolor{\revisionTwo}{the following} improvements \textcolor{\editageTwo}{are required:} parameter tuning is dependent \textcolor{\editage}{on} the steganalyst, \textcolor{\editageTwo}{e.g., the} training epochs, learning rate, and size of the training set; the \textcolor{\editageTwo}{failure} to detect hidden messages \textcolor{\editageTwo}{cannot be corrected} by the steganalyst. However, we expect that the future development of such techniques will resolve the limitations. According to Alvarez and Salzmann{~\cite{alvarez2016learning}}, the numbers of layers and neurons of deep networks can be determined using an additional \textcolor{\ediThree}{class of methods,} sparsity regularization, to the objective function.
The sizes of vectors of grouped parameters of each neuron in each layer \textcolor{\editage}{incur penalties} if the loss converges. The \textcolor{\editageTwo}{affected} neurons are removed if the neurons are assigned \textcolor{\editageTwo}{a value of} zero.

\section*{Acknowledgments}
This work was supported by the National Research Foundation of Korea (NRF) grant funded by the Korea government (Ministry of Science and ICT) [2014M3C9A3063541, 2018R1A2B3001628], and the Brain Korea 21 Plus Project in 2018.

\bibliographystyle{ws-procs11x85}
\bibliography{ws-pro-sample}

\end{document}